%% file: infocom_main.tex
\documentclass[journal]{IEEEtran}
\usepackage{amsfonts}
\usepackage{cite,graphicx,amsmath,amsthm}
\usepackage{fancyhdr}
\usepackage{dsfont}
\usepackage{array,color}
\usepackage{bm}
\usepackage{float}
\usepackage{algpseudocode}
\usepackage{multirow}
\usepackage{booktabs}
\usepackage{makecell}
\allowdisplaybreaks[4]
\usepackage{amsmath,amssymb, graphicx, epstopdf,cite,enumerate,booktabs, setspace, float,stfloats,bm, multirow, lscape, caption, subcaption, graphbox, soul, color, xcolor, hyperref}

\usepackage{diagbox}
\hypersetup{hidelinks}
\usepackage{cite, amsmath,amssymb, bm}
\usepackage{bbding}
\usepackage{amsthm}
\usepackage[normalem]{ulem}
\usepackage[linesnumbered,ruled]{algorithm2e}

\newtheorem{lemma}{Lemma}

\newtheorem{proposition}{Proposition}

\newtheorem{corollary}{Corollary}

\newtheorem{property}{Property}

\newtheorem{remark}{Remark}

\newtheorem{claim}{Claim}

\begin{document}

\title{\huge Green Robotic Mixed Reality with Gaussian Splatting}

\author{
Chenxuan Liu, He Li, Zongze Li, Shuai Wang$^{\dag}$, Wei Xu, Kejiang Ye,\\Derrick Wing Kwan Ng,~\emph{Fellow, IEEE}, and 
Chengzhong Xu,~\emph{Fellow, IEEE}
\thanks{
This work was supported by the National Natural Science Foundation of China (Grant No. 62371444), and the Shenzhen Science and Technology Program (Grant No. RCYX20231211090206005 and RCBS20231211090517022).

Chenxuan Liu, Shuai Wang, and Kejiang Ye are with the Shenzhen Institutes of Advanced Technology, Chinese Academy of Sciences, Shenzhen, China.
Chenxuan Liu is also with the University of Chinese Academy of Sciences. 
He Li and Chengzhong Xu are with the State Key Laboratory of IOTSC, Department of Computer and Information Science, University of Macau, Macau, China.
Zongze Li is with Peng Cheng Laboratory, Shenzhen, China.
Wei Xu is with the Manifold Tech Limited, Hong Kong, China.
Derrick Wing Kwan Ng is with the School of Electrical Engineering and Telecommunications, the University of New South Wales, Australia.

Corresponding author: Shuai Wang ({\tt\footnotesize s.wang@siat.ac.cn}). 
}
}
\maketitle
\thispagestyle{empty}
\pagestyle{empty}

\begin{abstract}
Realizing green communication in robotic mixed reality (RoboMR) systems presents a challenge, due to the necessity of uploading high-resolution images at high frequencies through wireless channels. This paper proposes Gaussian splatting (GS) RoboMR (GSRMR), which achieves a lower energy consumption and makes a concrete step towards green RoboMR. 
The crux to GSRMR is to build a GS model which enables the simulator to opportunistically render a photo-realistic view from the robot's pose, thereby reducing the need for excessive image uploads. 
Since the GS model may involve discrepancies compared to the actual environments, a GS cross-layer optimization (GSCLO) framework is further proposed, which jointly optimizes content switching (i.e., deciding whether to upload image or not) and power allocation across different frames. 
The GSCLO problem is solved by an accelerated penalty optimization (APO) algorithm. 
Experiments demonstrate that the proposed GSRMR reduces the communication energy by over $10$x compared with RoboMR. 
Furthermore, the proposed GSRMR with APO outperforms extensive baseline schemes, in terms of peak signal-to-noise ratio (PSNR) and structural similarity index measure (SSIM).
\end{abstract}

\section{Introduction}

Robotic mixed reality (RoboMR) \cite{morra2019building,delmerico2022spatial,li2024seamless}, 
which facilitates interactions between real robots and virtual agents in a shared symbiotic world, emerges as a promising solution to train and test robot learning systems \cite{li2021igibson}.
However, RoboMR requires synchronization of digital twins, which involves transmission of high frequency high resolution sensor data (e.g., images). 
Therefore, establishing green (i.e. low-energy) communication between the server and robots is non-trivial \cite{li2021igibson,morra2019building,delmerico2022spatial,li2024seamless,guo2020adaptive, yu2004iterative,zheng2016wireless,bastug2017toward,chen2018vr}.
Yet, achieving low energy consumption is crucial to extending the operation time of RoboMR systems.

Existing RoboMR resource allocation approaches \cite{bastug2017toward,chen2018vr,Cakir2024IDTC,Cheng2024resource,Zhao2024Adaptive,Pan2023joint} fail in fully exploiting the memory information from previous RoboMR executions.
In addition, most results ignore the inter-dependency between high-level data encoding/decoding at the application layer and low-level communication design at the physical layer.
Consequently, existing methods lead to excessive communication energy during the RoboMR interaction.

This paper proposes the Gaussian splatting (GS) RoboMR paradigm, termed GSRMR, which reduces the energy consumption by orders of magnitude compared to the current RoboMR system.
GSRMR leverages a collection of historical images recorded at previous RoboMR executions, and to build a GS model that can opportunistically generate a photo-realistic view from the robot’s pose, thus avoiding excessive uploading of images. 
Nonetheless, the GS model may involve discrepancies compared to the actual environments, due to variations of illuminations and changes of objects. 
To this end, we further propose a GS cross-layer optimization (GSCLO) framework to jointly optimize content switching (i.e., deciding to upload image or not) and power allocation across different frames. 
The GSCLO problem is solved by an accelerated penalty optimization (APO) algorithm, which is shown to have a lower complexity than search-based algorithms. 

We implement RoboMR in the Robot Operation System (ROS) and evaluate the performance gain brought by GSRMR and APO compared to existing RoboMR systems \cite{delmerico2022spatial,li2024seamless} and resource allocation algorithms \cite{yu2004iterative,zheng2016wireless,bastug2017toward}. Specifically, the energy saving brought by the proposed GSRMR is over $90\,\%$ compared with RoboMR. 
The peak signal-to-noise ratio (PSNR) of GSRMR with APO is at least $6$\,dB higher than other GSRMR benchmarks.
It is found that the GS model indeed helps VR synchronization, and the proposed GS-RoboMR effectively avoids GS discrepancies.

The remainder of the paper is organized as follows. Section \ref{section2} presents the problem formulation of RoboMR. Section \ref{section3} presents the GSRMR framework and section \ref{section4} presents the APO algorithm. Experiments are demonstrated and analyzed in Section \ref{section5}. 
Finally, conclusions are drawn in Section \ref{section6}.

\begin{figure}[t]
    \centering
    \includegraphics[width=0.48\textwidth]{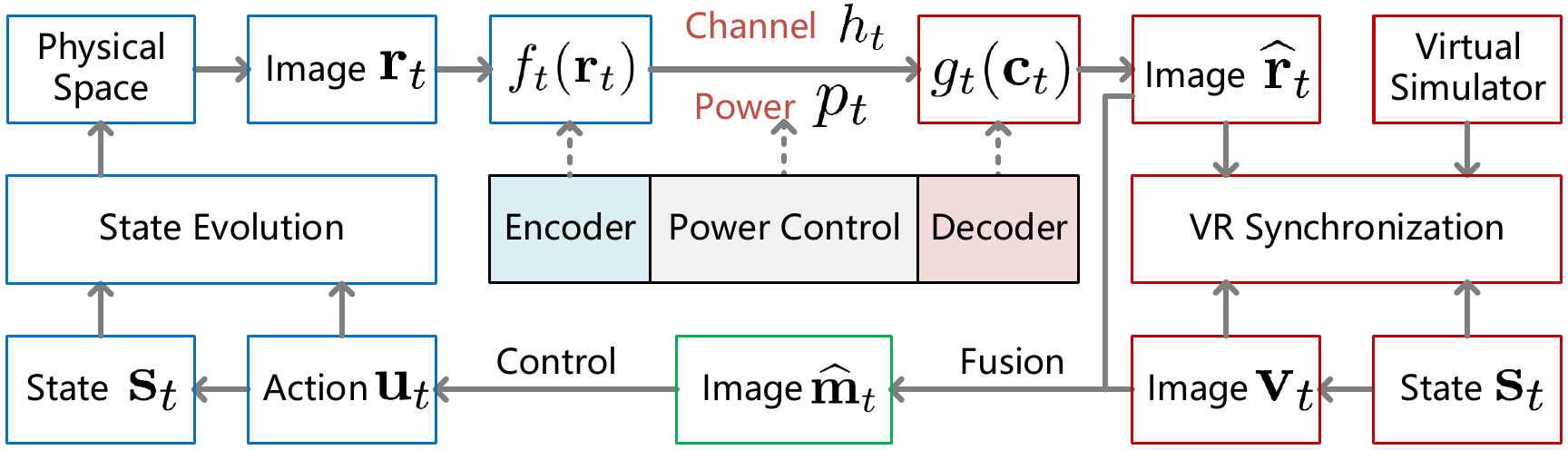}
    \caption{System model of RoboMR.}
    \label{fig:model}
    \vspace{-0.1in}
\end{figure}

\section{Problem Formulation}\label{section2}

\subsection{RoboMR System}
We consider a RoboMR system shown in Fig.~1, which consists of a mobile robot acting as a real agent and a simulation server rendering the virtual world. 
The system operation is divided into $T$ time slots, and the duration between consecutive slots is $\tau$.
At the $t$-th time slot, $t \in\{1, \ldots, T\}$, 
the robot state is ${\mathbf{s}}_t=({a}_{t},{b}_{t},{\theta}_{t})$, where $(a_t,b_t)$ and $\theta_t$ denote the position and orientation, respectively.
The robot control is $\mathbf{u}_t=(v_t,\phi_t)$, with $v_t$ and $\phi_t$ representing the linear and angular velocities, respectively.

The control vector $\mathbf{u}_t$ in RoboMR needs to avoid collision with both virtual and real agents in the MR image $\mathbf{m}_t$, which is given by
\begin{align}
\mathbf{m}_t= 
\mathbf{d}_t\odot \mathbf{r}_t+
    (\mathbf{1}-\mathbf{d}_t)
\odot\mathbf{v}_t, \label{mt}
\end{align}
where $\mathbf{r}_t\in\mathbb{R}^{lwc}$ and $\mathbf{v}_t\in\mathbb{R}^{lwc}$ are real and virtual camera images generated at the robot (sensing) and the server (simulating), respectively, with $l$, $w$, $c$ begin the length, width, and number of channels of images. 
Binary vector $\mathbf{d}_t$ is the mask of a real image for extracting the background and $(\mathbf{1}-\mathbf{d}_t)$ is the mask of a virtual image for extracting the virtual object. 
Symbol $\odot$ is the Hadamard product for element-wise vector multiplication. 


After the execution of $\mathbf{u}_t$ based on $\mathbf{m}_t$, the subsequent state is 
${\mathbf{s}_{t + 1}} = {\mathbf{s}_t} + F({\mathbf{s}_t},{\mathbf{u}_t})\tau$, where $F({\mathbf{s}_t},{\mathbf{u}_t})$ is the discrete-time kinematic model of the robot \cite{han2023rda}.
This completes one loop of RoboMR and the time index is updated as $t\leftarrow t+1$. 
The entire procedure terminates when $t=T$, and the robot learning dataset of the RoboMR is $\{\mathbf{m}_t,\mathbf{u}_t\}_{t=1}^T$.

\subsection{Green RoboMR System}

To generate $\mathbf{m}_t$ in \eqref{mt}, $\{\mathbf{r}_t\}_{t=1}^T$ needs to be uploaded over wireless communication, as the robot and simulator are physically isolated. 
This requires a pair of image encoder and decoder functions $(f_t,g_t)$, where the encoded vector $\mathbf{c}_t=f_t(\mathbf{r}_t)$ and the decoded vector
$\widehat{\mathbf{r}}_t=g_t(\mathbf{c}_t)$, 
such that $\widehat{\mathbf{r}}_t$ is close to $\mathbf{r}_t$.
Accordingly, the MR image is 
\begin{align}
\widehat{\mathbf{m}}_t= 
\mathbf{d}_t\odot \widehat{\mathbf{r}}_t+
    (\mathbf{1}-\mathbf{d}_t)
\odot\mathbf{v}_t. \label{mthat}
\end{align}
This $\widehat{\mathbf{m}}_t$ may differ from the desired image $\mathbf{m}_t$. 
To measure their difference, we adopt loss function 
 \begin{align}\label{LGS}
\mathcal{L}\left(\mathbf{m}_t,\widehat{\mathbf{m}}_t\right)
=(1-\lambda)\|\mathbf{m}_t-\widehat{\mathbf{m}}_t\|_1 +
  \lambda \mathcal{L}_{\mathsf{DSSIM}}(\mathbf{m}_t,\widehat{\mathbf{m}}_t),
\end{align}
where the DSSIM function $\mathcal{L}_{\mathsf{DSSIM}}$ is 
\begin{align}
\mathcal{L}_{\mathsf{DSSIM}}(\mathbf{m}_t,\widehat{\mathbf{m}}_t)
=
\frac{1}{1-\mathcal{L}_{\mathsf{SSIM}} (\mathbf{m}_t,\widehat{\mathbf{m}}_t)},
\end{align}
with $\mathcal{L}_{\mathsf{SSIM}}$ being the structural desimilarity index measure (SSIM) function detailed in \cite[Eqn. 5]{wang2011ssim}, 
and the weight $\lambda\in[0,1]$ is set to $\lambda=0.2$ according to \cite{kerbl20233d}.
To guarantee the fidelity of RoboMR, we must have 
$\frac{1}{T}\sum_{t=1}^T\mathcal{L}(\mathbf{m}_t,\widehat{\mathbf{m}}_t)\leq L_{\mathrm{th}}$, where $L_{\mathrm{th}}$ is a threshold to guarantee the desired quality of images. 
\footnote{To ensure satisfactory image qualities, the PNSR should be controlled above $25$\,dB  \cite{kerbl20233d}, which corresponds to a loss threshold of $L_{\mathrm{th}}$ below $0.05$.}

To guarantee successful transmission of $\mathbf{c}_t$, we must have $\tau R_{t}(p_{t})\geq D_t, \forall t$, where $R_{t}$ denotes the uplink achievable rate in bits/s, and $D_t$ in bits denotes the data volume of $\mathbf{c}_t$. 
Let $h_{t}\in\mathbb{C}$ and $|h_{t}|^2$ denote the uplink channel coefficient and gain from the robot to the server at time slot $t$. 
The rate $R_{t}$ at time slot $t$ is 
\begin{align}
R_{t}(p_{t})=B\mathrm{log}_2\left(1+\frac{|h_{t}|^2p_{t}}{\sigma^2}\right),
\end{align}
where $p_{t}\geq 0$ (with $\mathbf{p}=[p_{1},\cdots,p_{T}]^T$) and $B$ denote the transmit power and bandwidth of robot at time slot $t$, respectively, and $\sigma^2$ denotes the power of interference plus additive white Gaussian noise (AWGN) at the server.

The communication energy is minimized as follows:
\begin{align}
\mathsf{P}:\min_{\mathbf{p}\in\mathcal{P}, \{f_t,g_t\}}E(\mathbf{p}), \
\mathrm{s.t.}~
\frac{1}{T}\sum_{t=1}^T\mathcal{L}(\mathbf{m}_t,\widehat{\mathbf{m}}_t)\leq L_{\mathrm{th}},
\end{align}
where $E(\mathbf{p})=\tau\sum_{t=1}^Tp_t$ and $\mathcal{P}=\{\tau R_{t}(p_{t})\geq D_t, \, p_t\geq 0, \, \forall t \}$.
Note that $\widehat{\mathbf{m}}_t$ and $D_t$ are functions of $(f_t,g_t)$.

Existing image processing methods solve $\mathsf{P}$ by designing $\{f_t,g_t\}$ via compression or caching \cite{guo2020adaptive}. 
However, $D_t$ is still in the range of thousands of bytes. 
Existing power allocation methods, e.g., water-filling \cite{yu2004iterative}, max-min fairness \cite{zheng2016wireless}, solve $\mathsf{P}$ by adjusting $\mathbf{p}$ according to different channels. 
These approaches ignore the requirements of RoboMR.
Note that all above methods fail in exploiting the memory information from previous RoboMR executions. This implies that a lower energy is achievable for RoboMR.

\section{Gaussian Splatting RoboMR}\label{section3}

\begin{figure}[t]
    \centering
    \includegraphics[width=0.45\textwidth]{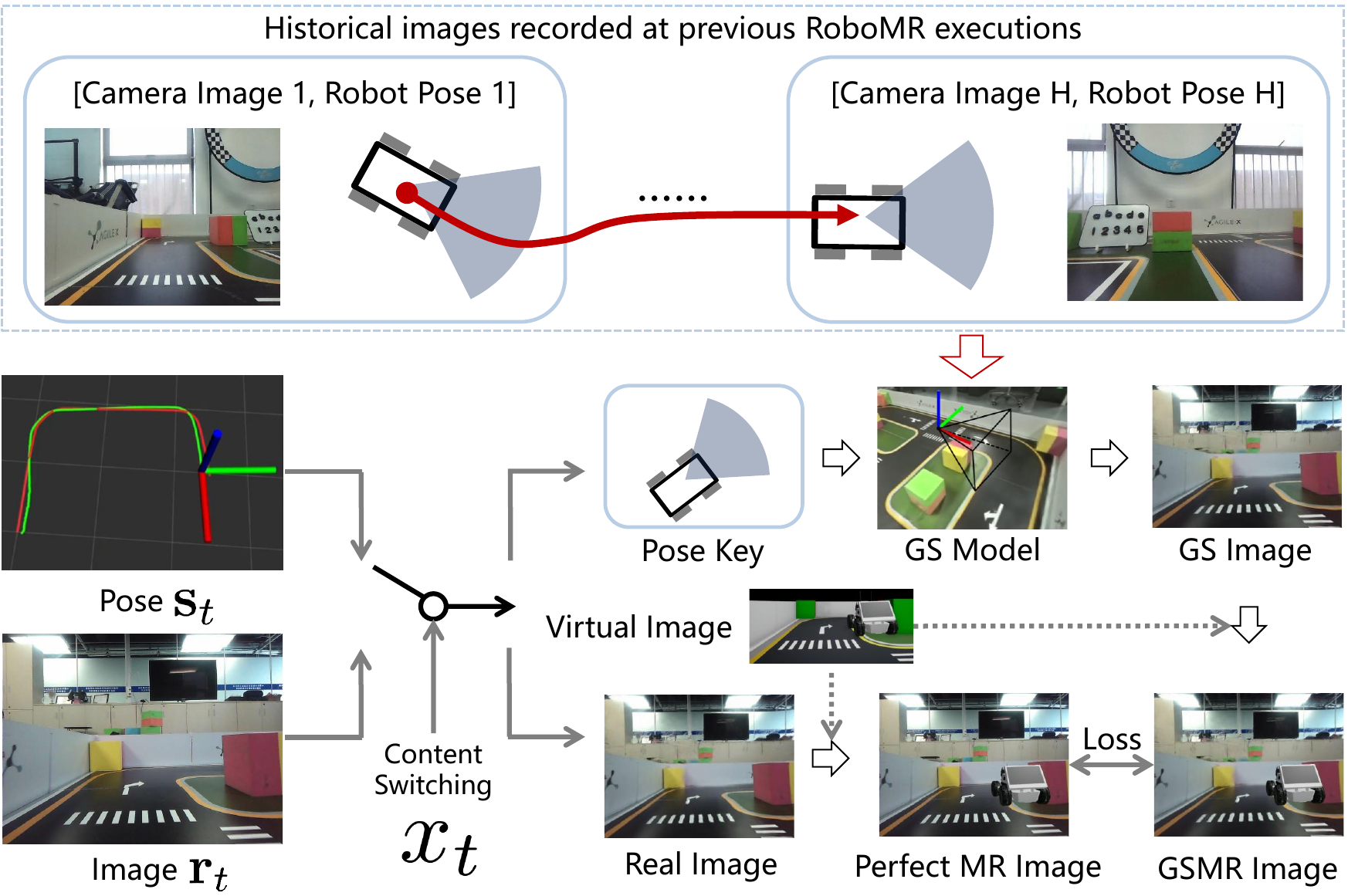}
    \caption{Architecture of GS-RoboMR.}
    \label{fig:GS}
    \vspace{-0.15in}
\end{figure}

This paper proposes a memory-assisted RoboMR paradigm, termed GSRMR. As shown in Fig.~2, the memory is a collection of $H$ historical images recorded at previous RoboMR executions. Leveraging these memories, we build a GS model that can render a photo-realistic image from a new viewpoint adjacent to the recording position. 
In particular, the GS model, $\Phi$, is operated on a given robot pose $\mathbf{s}_t$, and outputs a synthesis image $\mathbf{y}_t=\Phi(\mathbf{s}_t)$. 
If $\mathbf{y}_t$ is close to $\mathbf{r}_t$, then the robot only needs to upload a pose \emph{key} $\mathbf{s}_t$ to inform its viewpoint, without any image uploading.
This key is only tens of bytes and is orders of magnitude smaller than the image data volume.

However, the GS model may involve discrepancies compared to the actual environments (i.e., $\mathbf{y}_t$ mismatches $\mathbf{r}_t$), due to the variations in illuminations and changes of objects. 
To avoid such discrepancies, we further propose a content switching scheme shown in Fig.~\ref{fig:GS}.
Specifically, the proposed method adopts 
	\begin{align}\label{E}
		f_t(\mathbf{r}_t) = x_t\mathbf{r}_t + (1-x_t)\mathbf{s}_t, 
	\end{align}
where $x_t\in\{0,1\}$ (with $\mathbf{x}=[x_1,\cdots,x_T]^T$) denotes the content switching variable to be optimized, $x_t=1$ represents transmitting image, and $x_t=0$ represents transmitting pose. 
As such, the data volume becomes $D_t=x_t I + (1-x_t)S$, where $I$ and $S$ in bits (with $S\ll I$) are the data volume of image and pose, respectively. 
Accordingly, the decoder function adopts 
\begin{align}\label{D}
		g_t(\mathbf{c}_t)=x_t\mathbf{c}_t + (1-x_t)\Phi(\mathbf{c}_t).
\end{align}

Putting \eqref{E} and \eqref{D} into $\mathsf{P}$, the energy minimization problem under GS-RoboMR is formulated as 
	\begin{subequations}
		\label{P2}
		\begin{align}
			\mathsf{P}_{\mathrm{GS}}:\,\,\,
			\min_{\mathbf{p},\mathbf{x}
			}~&\tau\sum_{t=1}^{T}p_{t} \label{P2a}  \\
			\textrm{s.t.} ~~ & \tau R_{t}(p_{t})\geq x_t I + (1-x_t)S,\ \forall t, \label{Pc_GS}
 \\  
     &
\frac{1}{T}\sum_{t=1}^TL_t
(1-x_t)\leq L_{\mathrm{th}},
   \label{Pd_GS}
   \\
   & p_{t}\geq 0, \ x_t\in\{0,1\}, \ \forall t, \label{Px_GS}
		\end{align}
	\end{subequations}
where we have defined 
\begin{align}
L_t:=\mathcal{L}\Big(
&
\mathbf{d}_t\odot\mathbf{r}_t+
(\mathbf{1}-\mathbf{d}_t)
\odot\mathbf{v}_t, 
\nonumber\\
&
\mathbf{d}_t\odot \Phi(\mathbf{s}_t) +
(\mathbf{1}-\mathbf{d}_t)
\odot\mathbf{v}_t
\Big).
\end{align}

Problem $\mathsf{P}_{\mathrm{GS}}$ is defined as the GSCLO problem, since it leverages the GS model to design cross-layer variables $\{x_t,p_t\}$.
Note that $\mathsf{P}_{\mathrm{GS}}$ is a mixed integer nonlinear programming (MINLP) problem, which can be solved by branch-and-bound (B\&B) \cite{diamond2016cvxpy}. 
However, the complexity is exponential in $T$, which is time-consuming when $T$ is large. 
On the other hand, $\mathsf{P}_{\mathrm{GS}}$ can be addressed by continuous relaxation plus rounding \cite{hubner2014rounding}, but this would lead to non-negligible performance loss. Therefore, we develop a fast near-optimal algorithm for solving $\mathsf{P}_{\mathrm{GS}}$ in Section~\ref{section4}.

\section{Accelerated Penalty Optimization}\label{section4}

\subsection{Penalty for Binary Constraints}

To tackle the discontinuity, we first relax the binary constraint $x_{t}\in\{0,1\}$ into a linear constraint $0 \leq x_{t} \leq 1$, $\forall t$. 
However, the relaxation is not tight. 
To promote a binary solution for the relaxed variable $\{x_{t}\}$, we augment the objective function with a penalty term as in \cite{rinaldi2009new}. The associated reformulation with regularized penalty term of $\mathsf{P_{GS}}$ is given by
	\begin{subequations}
		\label{P3}
		\begin{align}
			\mathsf{P}_{\mathrm{Penalty}}:\,\,\,
			\min_{\mathbf{p}\succeq \mathbf{0},\ \mathbf{0} \preceq \mathbf{x} \preceq \mathbf{1}
			}~&\tau\sum_{t=1}^{T}p_{t} + \varphi(\mathbf x) \label{PPenaltya}  \\
			\textrm{s.t.} ~~ & 
            \textsf{constraints }(\ref{Pc_GS}), (\ref{Pd_GS}),
		\end{align}
	\end{subequations}
where $\varphi(\mathbf x)$ is a penalty function to penalize the violation of the zero-one integer constraints. A celebrated penalty function was introduced in \cite{lucidi2010exact}, where the penalty function is set as
\begin{align} \label{eq:penaltyterm}
    \varphi(\mathbf x)=\frac{1}{\beta}\sum_{t=1}^T x_{t}(1-x_t),
\end{align}
where $\beta>0$ is the penalty parameter.
According to \cite[Proposition 1]{lucidi2010exact}, with the penalty term in (\ref{eq:penaltyterm}), there exists an upper bound $\bar\beta>0$ such that for any $\beta\in[0,\bar\beta]$, $\mathsf{P}_{\mathrm{Penalty}}$ and $\mathsf{P_{GS}}$ have the same minimum points, i.e., $\mathsf{P}_{\mathrm{Penalty}}$ and $\mathsf{P_{GS}}$ are equivalent with a proper choice of $\beta$ \cite{lucidi2010exact}.

\subsection{Difference of Convex Algorithm}

In problem $\mathsf{P}_{\mathrm{Penalty}}$, the constraints in \eqref{Pc_GS} are convex, since the function $-\mathrm{log}_2(\cdot)$ is convex.
Besides, the constraints in \eqref{Pd_GS} are all affine. 
The only nonconvex term in the objective function is the penalty $\varphi(\mathbf x)$. 
However, by expanding $\varphi(\mathbf x)$ as 
$\varphi(\mathbf x) =
\frac{1}{\beta}\sum_{t=1}^T x_{t}- \frac{1}{\beta}\sum_{t=1}^T x_{t}^2
$, we find that $\varphi(\mathbf x)$ is difference of convex (DC) function \cite{an2005dc}, and such DC functions enjoy excellent surrogate properties. 
Specifically, given a certain solution $\{x_{t}^{\star}\}$,
we apply the first-order Taylor expansion on $\varphi(\mathbf x)$ and obtain 
$\varphi(\mathbf x)\approx \widehat{\varphi}(\mathbf x|\mathbf x^\star )$, where 
\begin{align}
&\widehat{\varphi}(\mathbf x|\mathbf x^\star ) = 
\sum_{t=1}^T 
\left(
\frac{1}{\beta}x_{t}-\frac{2}{\beta}x_{t}^{\star}x_{t}+\frac{1}{\beta}x_{t}^{\star^2}
\right).
\end{align}
Using the properties of DC functions, the following proposition can be established.

\begin{proposition}
$\widehat{\varphi}(\mathbf x|\mathbf x^\star )$ satisfy the following conditions:

\noindent(i) Upper bound: $\widehat{\varphi}(\mathbf x|\mathbf x^\star )
\geq \varphi(\mathbf x)$ for any $\mathbf x$.

\noindent(ii) Convexity: $\widehat{\varphi}(\mathbf x|\mathbf x^\star )$ is convex in $\mathbf{x}$.

\noindent(iii) 
Equality:
$[\widehat{\varphi}(\mathbf x^\star|\mathbf x^\star ), \nabla\widehat{\varphi}(\mathbf x^\star|\mathbf x^\star )]
=[\varphi(\mathbf x^\star),\nabla\varphi(\mathbf x^\star)]$.
\end{proposition}
\begin{proof}
Part (i) is proved by checking 
\begin{align}
\widehat{\varphi}(\mathbf x|\mathbf x^\star )-\varphi(\mathbf x)
&=
\frac{1}{\beta}\sum_{t=1}^T x_{t}^2
-
\sum_{t=1}^T\frac{2}{\beta}x_{t}^{\star}x_{t}
+\sum_{t=1}^T
\frac{1}{\beta}x_{t}^{\star^2}
\nonumber
\\
&=
\frac{1}{\beta}\sum_{t=1}^T
\left(x_{t}-x_{t}^{\star}\right)^2\geq 0.
\end{align}
Part (ii) is proved by checking the semi-definiteness of the Hessian of $\widehat{\varphi}$.
In particular, $\nabla^2_{\mathbf x}\widehat{\varphi}=\mathbf{0}$, which is semi-definite.
Part (iii) is proved by computing and comparing the function and gradient values of  $\widehat{\varphi}$ and $\varphi$.
\end{proof}
Based on part (i) of \textbf{Proposition 1}, an upper bound problem, denoted as $\mathsf{P}_{\mathrm{Penalty}}'$, can be directly
obtained if we replace the function $\varphi(\mathbf x)$ by $\widehat{\varphi}(\mathbf x|\mathbf x^\star )$ expanded around a feasible point $\mathbf x^\star$ for the problem $\mathsf{P}_{\mathrm{Penalty}}$.
Based on part (ii) of \textbf{Proposition 1}, this $\mathsf{P}_{\mathrm{Penalty}}'$ is guaranteed to be a solvable convex problem.
Moreover, according to part (iii) of \textbf{Proposition 1}, a tighter upper bound can be achieved if we treat the obtained solution as another feasible point and continue
to construct the next round surrogate function. 
This leads to the DC algorithm \cite{an2005dc}, which solves a sequence of convex optimization problems 
$\{\mathsf{P}_{\mathrm{DC}}^{[0]},\mathsf{P}_{\mathrm{DC}}^{[1]},\mathsf{P}_{\mathrm{DC}}^{[2]},\cdots\}$, where $\mathsf{P}_{\mathrm{DC}}^{[n+1]}$ is the problem in the $(n+1)$-th iteration of the DC algorithm, and is given by 
		\begin{align}\label{Pt+1}
			\mathsf{P}_{\mathrm{DC}}^{[n+1]}:\,\,\,
			\min_{\mathbf{p}\succeq \mathbf{0},\ \mathbf{0} \preceq \mathbf{x} \preceq \mathbf{1}
			}~&\tau\sum_{t=1}^{T}p_{t} + 
   \widehat{\varphi}(\mathbf x|\mathbf x^{[n]} ) \nonumber
   \\
			\textrm{s.t.} ~~ & \textsf{constraints }(\ref{Pc_GS}), (\ref{Pd_GS}).
		\end{align}
Here, $\mathbf{x}^{[n]}=[x_{1}^{[n]},\cdots,x_{T}^{[n]}]^T$ is the optimal solution of $\mathbf{x}$ to $\mathsf{P}_{\mathrm{DC}}^{[n]}$.
Each $\mathsf{P}_{\mathrm{DC}}^{[n+1]}$ is a convex problem and can be solved via off-the-shelf toolbox (e.g., CVXPY) with a complexity of $\mathcal{O}((2T)^{3.5})$.
According to \textbf{Proposition 1} and \cite{abbaszadehpeivasti2024rate}, any limit point of the sequence $\{(\mathbf{p}^{[n]}, \mathbf{x}^{[n]}\}_{n=0,1,\cdots}$ is a Karush-Kuhn-Tucker (KKT) solution to the problem $\mathsf{P}_{\mathrm{Penalty}}$ as long as the starting point 
$(\mathbf{p}^{[0]}, \mathbf{x}^{[0]})$ is feasible. The total complexity of APO is thus $\mathcal{O}(N(2T)^{3.5})$, where 
$N$ is the number of iterations required for the APO to converge.

\subsection{Ranking-based Initialization}

The APO algorithm requires a feasible initialization. 
A naive initialization is $\{x_{t}=1,\forall t\}$. 
However, this easily leads to slow convergence of the DC algorithm. 
Here, we propose a ranking-based algorithm to accelerate the convergence while guaranteeing feasibility.
Specifically, the ranking algorithm sorts the list of GS losses $\{L_t\}_{t=1}^T$ in the ascending order, and the reordered list is denoted as $\{L_{w(1)}, L_{w(2)},\cdots\}$, where $w(j)$ is an index mapping $j\rightarrow w(j)$ such that 
$L_{w(j)}\leq 
L_{w(j+1)}$ for all $j$ and $w(1)$ is the index corresponding to the smallest value among all $\{L_t\}$. 
The ranking solution is given by:
\begin{equation}\label{ranking}
(x_t^\star,p_t^\star)=
\left\{\begin{array}{ll}
(0,\frac{\sigma^2}{|h_{t}|^2}
2^{
\frac{S}{\tau B}
}-\frac{\sigma^2}{|h_{t}|^2}), \ &\mathrm{if}~w(t)\leq \mu \\
(1,\frac{\sigma^2}{|h_{t}|^2}
2^{
\frac{I}{\tau B}
}-\frac{\sigma^2}{|h_{t}|^2}), \ &\mathrm{if}~w(t)> \mu
\end{array}\right.,
\end{equation} 
where 
\begin{align}
\mu=
\mathop{\mathrm{min}}
\left\{\sum_{t=1}^T x_t^\star:
\frac{1}{T}\sum_{t=1}^TL_t
(1-x_t^\star)\leq L_{\mathrm{th}}\right\}.
\end{align}

The insight of solution \eqref{ranking} is that $L_t$ represents how much loss is added when switching from image to pose, and $x_t=0$ should be assigned to the frame that guarantees minimum loss to avoid GS discrepancies. 
The following proposition can be established to confirm the above insight.

\begin{proposition}
The ranking-based solution $\{x_t^\star,p_t^\star\}$ is optimal to $\mathsf{P_{GS}}$ if $|h_1|^2=\cdots=|h_T|^2$.
\end{proposition}
\begin{proof}
First, it can be proved that constraint \eqref{Pc_GS} always activates at the optimal $\{x_t^*,p_t^*\}$.
This gives $p_t^* = \frac{\sigma^2}{|h_{t}|^2}
2^{
\frac{I}{\tau B}
}-\frac{\sigma^2}{|h_{t}|^2}$ if $x_t^*=1$ and 
$p_t^* = \frac{\sigma^2}{|h_{t}|^2}
2^{
\frac{S}{\tau B}
}-\frac{\sigma^2}{|h_{t}|^2}$ if $x_t^*=0$.
Putting the above equations and $|h_1|^2=\cdots=|h_T|^2$ into $\mathsf{P_{GS}}$, $\mathsf{P_{GS}}$ is equivalent to
	\begin{subequations}
		\begin{align}
			\min_{x_t\in\{0,1\},\forall t
			}~&\tau\left(\frac{\sigma^2}{|h_{t}|^2}
2^{
\frac{I}{\tau B}
}-\frac{\sigma^2}{|h_{t}|^2}\right)\sum_{t=1}^{T}x_{t} \label{P2a}  \\
			\textrm{s.t.} ~~ &
\frac{1}{T}\sum_{t=1}^TL_t
(1-x_t)\leq L_{\mathrm{th}}.
		\end{align}
	\end{subequations}
It can be seen that the objective function is linear in $\mu$.
Second, for any $(t,j)$ with $L_t\geq L_{j}$, the solution with $(x_t,x_j)=(1,0)$ always dominates that with $(x_t,x_j)=(0,1)$. With the two properties, the proof can be completed by contradiction.
\end{proof}

\textbf{Proposition 2} shows that $\{x_{t}^{[0]}=x_{t}^\star,p_t^{[0]}=p_t^\star\}$ still serves as a good initialization for APO. 
\textbf{Proposition 2} also shows that the performance of ranking-based algorithm depends on the variation of $\{h_t\}$. 
Particularly, if the equal channel gain assumption is not satisfied, the ranking-based algorithm is in general suboptimal to $\mathsf{P}_{\mathrm{GS}}$. 

\section{Experiments}\label{section5}

\begin{figure}[!t]
	\centering
	\begin{subfigure}{0.45\linewidth}
		\centering
		\includegraphics[width=\linewidth]{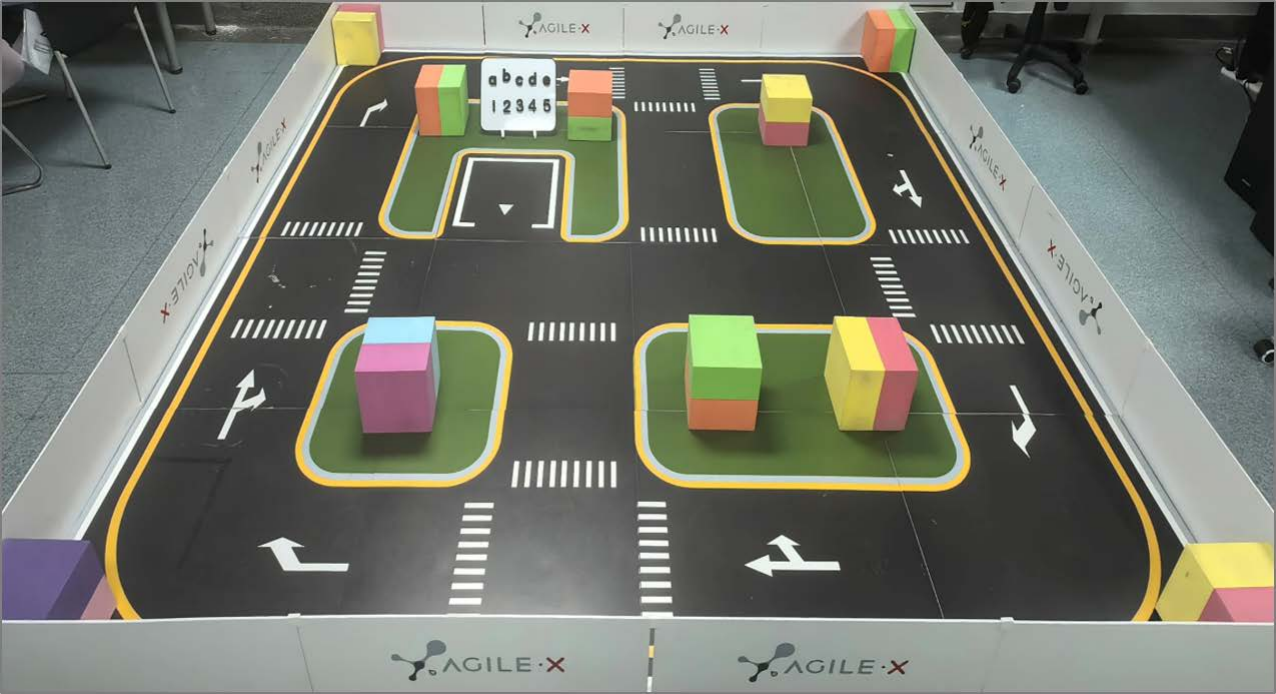}
		\caption{Real AGX.}
	\end{subfigure}
    	\begin{subfigure}{0.45\linewidth}
		\centering
		\includegraphics[width=\linewidth]{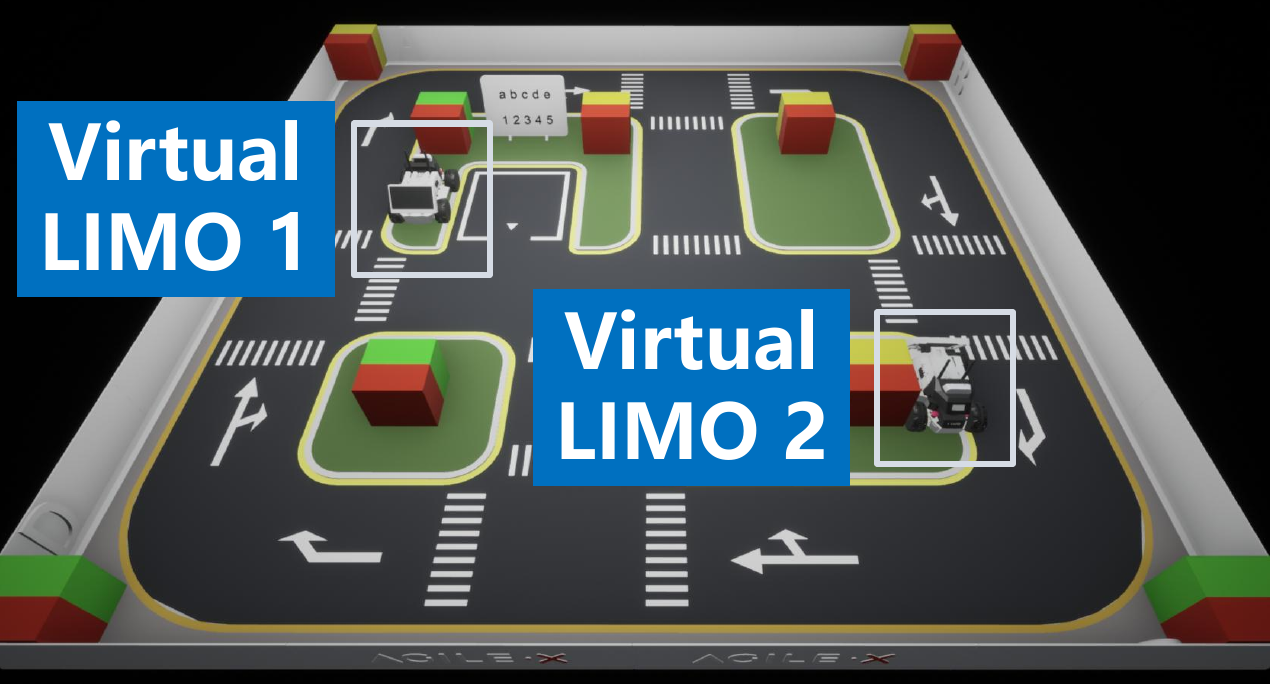}
		\caption{Virtual AGX.}
	\end{subfigure}
    	\begin{subfigure}{0.45\linewidth}
		\centering
		\includegraphics[width=\linewidth]{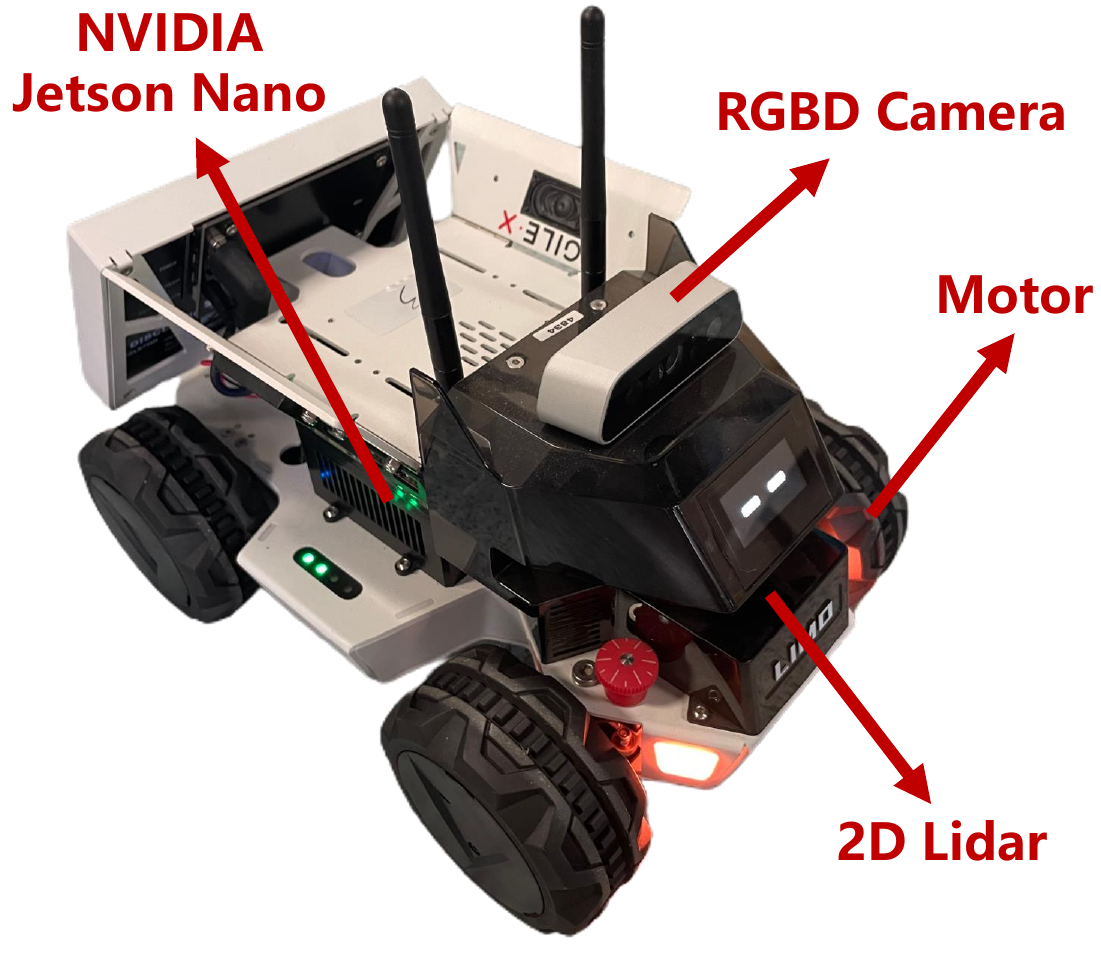}
		\caption{LIMO.}
	\end{subfigure}
    	\centering
	\begin{subfigure}{0.44\linewidth}
		\centering
		\includegraphics[width=\linewidth]{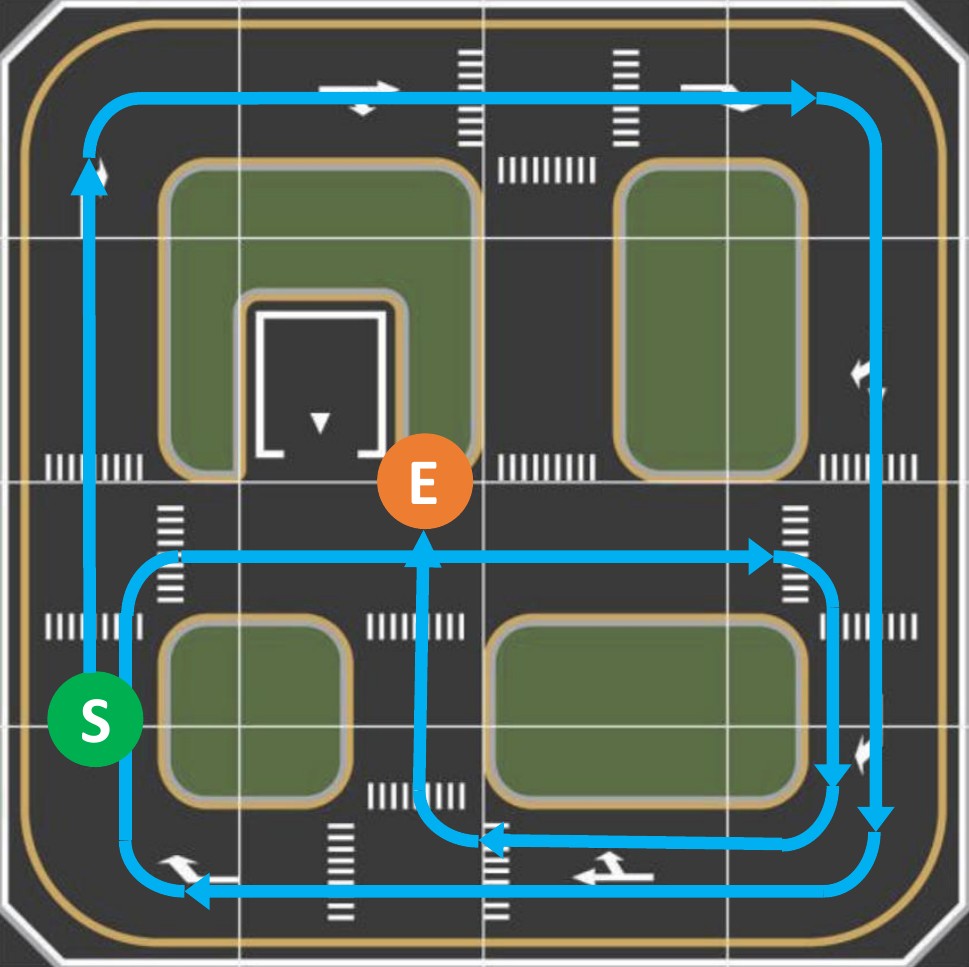}
      \vspace{-0.2in}
		\caption{Robot route.}
	\end{subfigure}
    \vspace{-0.05in}
	\caption{Implementation of the RoboMR platform. }
	\label{platforms}
    \vspace{-0.15in}
\end{figure}

We implement the RoboMR system exploiting C++ and Python in ROS. 
The real $4\,\text{m}\times 4\,\text{m}$ Agilex (AGX) sandbox platform is shown in Fig.~\ref{platforms}a.
The virtual world in Fig.~\ref{platforms}b is constructed using the CARLA simulator \cite{carla}, which acts as a digital twin of the real AGX platform, but two additional virtual objects (marked in white boxes) are added to the virtual world. 
The CARLA simulator is implemented on a Ubuntu workstation with a $3.7$\,GHZ AMD Ryzen 9 5900X CPU and an NVIDIA $3090$\,Ti GPU. 
We adopt a robot named LIMO shown in Fig.~\ref{platforms}c, which has a 2D lidar, an RGBD camera, and an onboard NVIDIA Jetson Nano computing platform. 
The robot route is marked as the blue line in Fig.~\ref{platforms}d, where green and orange boxes represent start and end positions.

We navigate the LIMO robot along the route for two rounds. 
In the first round, the robot collects $H$ frames (including images and poses) for training a GS model \cite{kerbl20233d}. 
In the second round, the robot collects $T$ frames for evaluation, with $\tau=0.1$\,s, $I=67.2$\,Kbits, $S=192$\,bits, and $T=288$.

The bandwidth $B=1$\,MHz and $\sigma^2=-60$\,dBm \cite{wang2020angle}.
The channel is assumed to be Rician fading, i.e., \cite{wang2020angle}
\begin{align}\label{ht}
&h_t=\sqrt{\varrho_0 d_0^{-\alpha}}\Big(\sqrt{\frac{K}{1+K}}\,g_t^{\mathrm{LOS}}
+\sqrt{\frac{1}{1+K}}\,g_t^{\mathrm{NLOS}}\Big),
\end{align}
where $\varrho_0=-30$\,dB is the pathloss at $1\,\mathrm{m}$, $d_0$ is computed based on the robot trajectory and the server position (around $10$\,m), and $\alpha=3$ is the pathloss exponent.
The Rician K-factor is set to $K=1$.
The LoS component is $g_t^{\mathrm{LOS}}=\mathrm{exp}\left(-\mathrm{j}\pi\,\mathrm{sin}\,\psi_t\right)$ with $\psi_t\in\mathcal{U}(-\pi,\pi)$ being the phase angle, and the non-LoS component is $g_t^{\mathrm{NLOS}}\sim\mathcal{CN}(0,1)$ \cite{wang2020angle}.

We compare GSRMR with APO to the following baselines: 1) \textbf{RoboMR}: Vanilla RoboMR \cite{li2024seamless} uploading all images (i.e., $\{x_t=1\}$); 2) \textbf{Ranking}: GSRMR with \eqref{ranking}; 3) \textbf{Rounding}: GSRMR by solving $\mathsf{P}_{\mathrm{GS}}$ with continuous relaxation plus rounding \cite{hubner2014rounding}
4) \textbf{Search}: GSRMR by solving $\mathsf{P}_{\mathrm{GS}}$ with iterative local search \cite{neumann2007randomized}; 5) \textbf{MaxRate}: GSRMR with water-filling power allocation \cite{yu2004iterative}; 6) \textbf{Fairness}: GSRMR with max-min fairness power allocation \cite{zheng2016wireless}; 7) \textbf{RoboGS}: Vanilla GS \cite{kerbl20233d} with no image uploading (i.e., $\{x_t=0\}$).

\begin{figure*}[!t]
	\centering
	\begin{subfigure}{0.47\linewidth} 
		\centering
		\includegraphics[width=\linewidth]{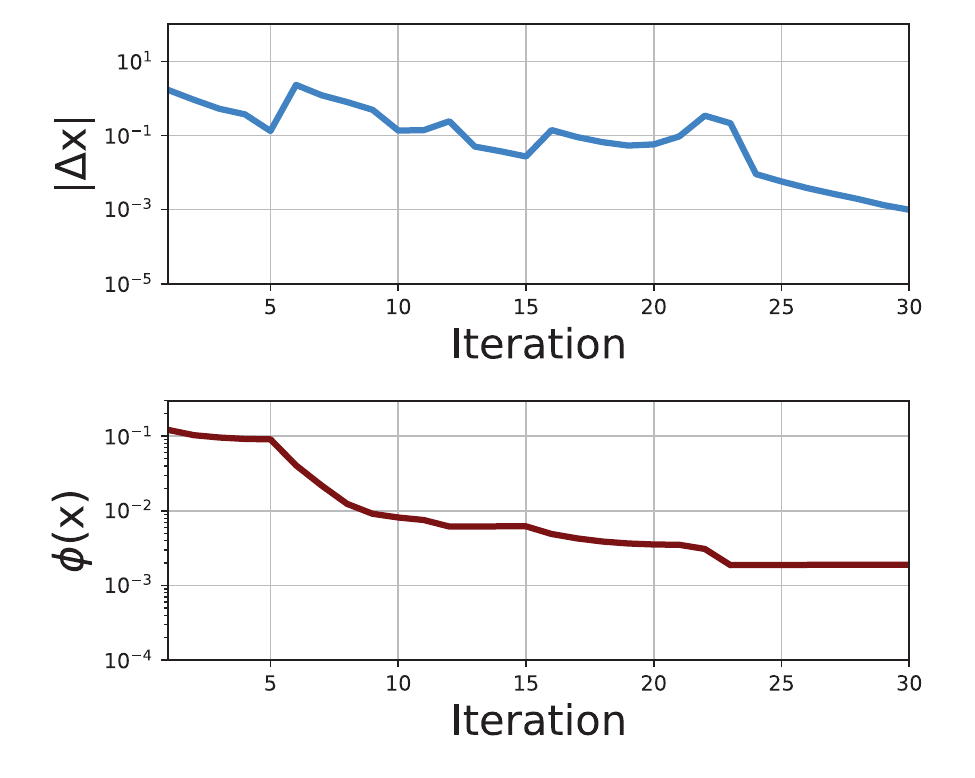}
      \vspace{-0.2in}
		\caption{$\|\Delta\mathbf{x}\|$ and $\phi(\mathbf{x})$ versus $n$ at $L_{\mathrm{th}}=0.03$.}
	\end{subfigure}
 	\begin{subfigure}{0.51\linewidth} 
		\centering
		\includegraphics[width=\linewidth]{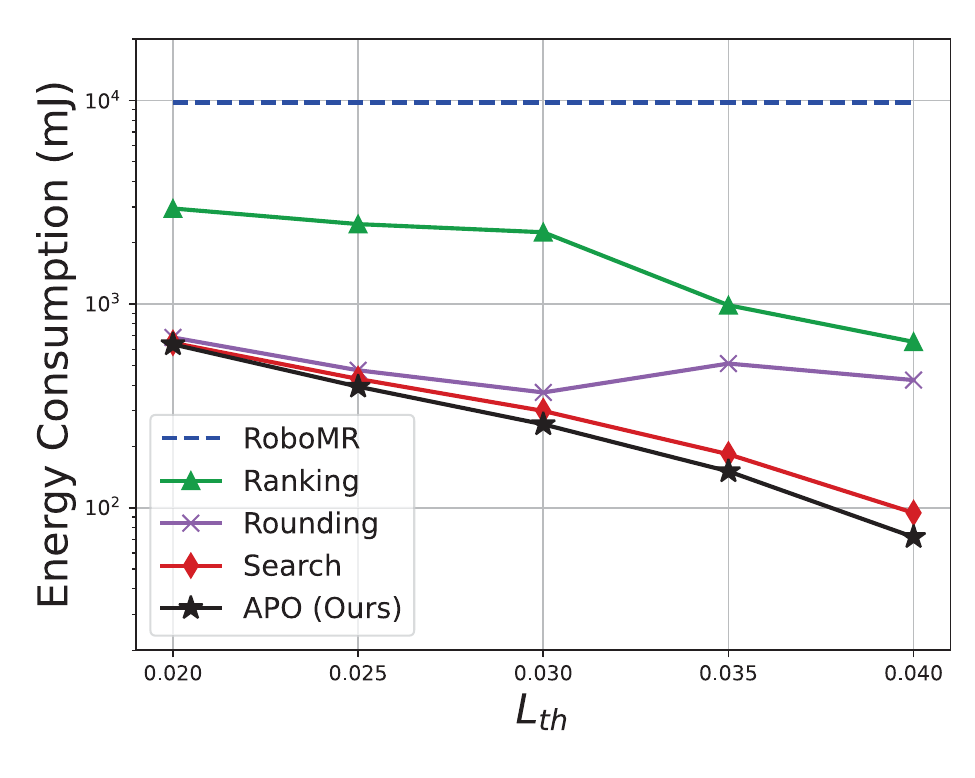}
      \vspace{-0.2in}
		\caption{Energy consumption versus $L_{\mathrm{th}}$.}
	\end{subfigure}
    \vspace{-0.05in}
     \caption{(a) Convergence behavior of APO. (b) Comparison of energy consumption.}
     \label{fig4}
\end{figure*}

\begin{table*}[!t]
    \centering
    \caption{Comparison between MaxRate, Fairness, and APO.}
    \vspace{-0.05in}
    \label{table1}
        \scalebox{0.95}{
    \begin{tabular}{ccccccccccccc}
        \toprule
        Metric & \multicolumn{3}{c}{Loss}  & \multicolumn{3}{c}{PSNR} &  \multicolumn{3}{c}{SSIM} \\
        \cline{2-4} \cline{5-7} \cline{8-10} 
        $L_{\mathrm{th}}$ & Fairness & MaxRate & APO (Ours) & Fairness & MaxRate & APO (Ours) & Fairness & MaxRate & APO (Ours) \\
        \midrule
        0.02 & $0.048$ & $0.04$ & $\mathbf{0.018}$ & $23.00$ & $27.69$ & $\mathbf{39.38}$ & $0.949$ & $0.958$ & $\mathbf{0.981}$ \\
        0.025 & $0.048$ & $0.044$ & $\mathbf{0.024}$ & $23.00$ & $25.29$ & $\mathbf{34.88}$ & $0.949$ & $0.954$ & $\mathbf{0.975}$ \\
        0.03 & $0.048$ & $0.047$ & $\mathbf{0.029}$ & $23.00$ & $23.41$ & $\mathbf{31.87}$ & $0.949$ & $0.950$ & $\mathbf{0.970}$ \\
        0.035 & $0.048$ & $0.048$ & $\mathbf{0.034}$ & $23.00$ & $23.11$ & $\mathbf{29.08}$ & $0.949$ & $0.949$ & $\mathbf{0.965}$ \\
        \bottomrule
    \end{tabular}
    }
    \vspace{-0.1in}
\end{table*}

First, we conduct numerical experiments to validate the convergence of the proposed APO algorithm.
Specifically, we consider the case of $L_{\mathrm{th}}=0.03$.
The variation between consecutive iterations $\|\Delta\mathbf{x}\|$ with $ \Delta\mathbf{x}=\mathbf{x}^{[n]}-\mathbf{x}^{[n-1]}$ versus the iteration count $n$ is shown in Fig.~\ref{fig4}a.
It can be seen that $\|\Delta\mathbf{x}\|$ falls below $10^{-3}$ after $30$ iterations.
This demonstrates the convergence of the proposed APO. 
The zero-one loss $\phi(\mathbf{x})=
\frac{1}{T}\sum_{t=1}^T x_{t}(1-x_t)$ versus the iteration count $n$ is also shown in Fig.~\ref{fig4}a. 
It can be seen that $\phi(\mathbf{x})$ is below $0.01$ after $10$ iterations, meaning that the APO algorithm converges to a point close to the desired binary solution. 
Based on this observation, we set $N=10$ in the subsequent experiments.

Next, to evaluate the energy saving brought by GS, we consider the case of $L_{\mathrm{th}}=\{0.02,0.025,0.03,0.035,0.04\}$.
The average energy consumption versus the MR loss requirement $L_{\mathrm{th}}$ is shown in Fig.~\ref{fig4}b. 
Compared to RoboMR, the energy saved by GSRMR is at least $10$\,dB, i.e., equal to $10$x power reduction.
Particularly, when the loss threshold is $0.04$, the energy saved by GSRMR is larger than $20$\,dB, i.e., equal to $100$x power reduction.
We refer to this quantity as \textbf{GS energy saving factor}, which is proportional to $L_{\mathrm{th}}$.

We also observe that as $L_{\mathrm{th}}$ increases, the energy consumption of all the simulated schemes are reduced, except for the Rounding scheme. This demonstrates the instability of the relax-and-round procedure for solving MINLP problems. 
Moreover, no matter how the loss threshold varies, the proposed GSRMR with APO consistently outperforms all the other schemes. 
Compared to the second-best scheme (i.e., the Search scheme), the energy reduction is over $20\%$ at $L_{\mathrm{th}}=0.04$. 
Their difference is more significant as $L_{\mathrm{th}}$ increases. 
This is because $\mathbf{x}$ is sparser under a slightly larger $L_{\mathrm{th}}$, and the Search method is prone to get stuck at local minimum. 
This implies that the proposed APO algorithm is more robust than the Search scheme.

To demonstrate the benefit brought by cross-layer optimization, we further compare APO with the MaxRate and Fairness schemes. 
These schemes represent the pure physical-layer scheme without considering MR requirements.\footnote{For fair comparison, we first execute the proposed GSRMR with APO and then obtain the associated energy consumption, which is used as the energy budget for the MaxRate and Fairness schemes.}
The quantitative results are shown in Table~\ref{table1}.
First, the proposed GSRMR with APO always achieves a qualified loss below $L_{\mathrm{th}}$.
Second, the proposed APO achieves a PSNR above $29$\,dB under all cases. 
Compared with the other schemes, the improvement of PSNR is at least $6$\,dB.
Lastly, the proposed GSRMR with APO outperforms all the other schemes in terms of the SSIM. 

Finally, qualitative comparisons with RoboMR and RoboGS at $L_{\mathrm{th}}=0.03$ are shown in Fig.~\ref{fig:demo}a.
It can be seen that images B and E are blurred under the RoboGS scheme. 
This demonstrates that the GS model involves discrepancies compared to the real-world environments.
In contrast, all the images of APO have high qualities and look similar to those of RoboMR (i.e., ground truth).
This corroborates Fig. \ref{fig:demo}b, where the RoboMR with APO opportunistically switches between uploading images and poses, and adjusts the robot transmit power. 
This implies that the GS prompts green RoboMR, but we need to carefully avoid GS discrepancies.

\begin{figure*}[!t]
	\centering
	\begin{subfigure}{0.69\linewidth} 
		\centering
		\includegraphics[width=1\linewidth]{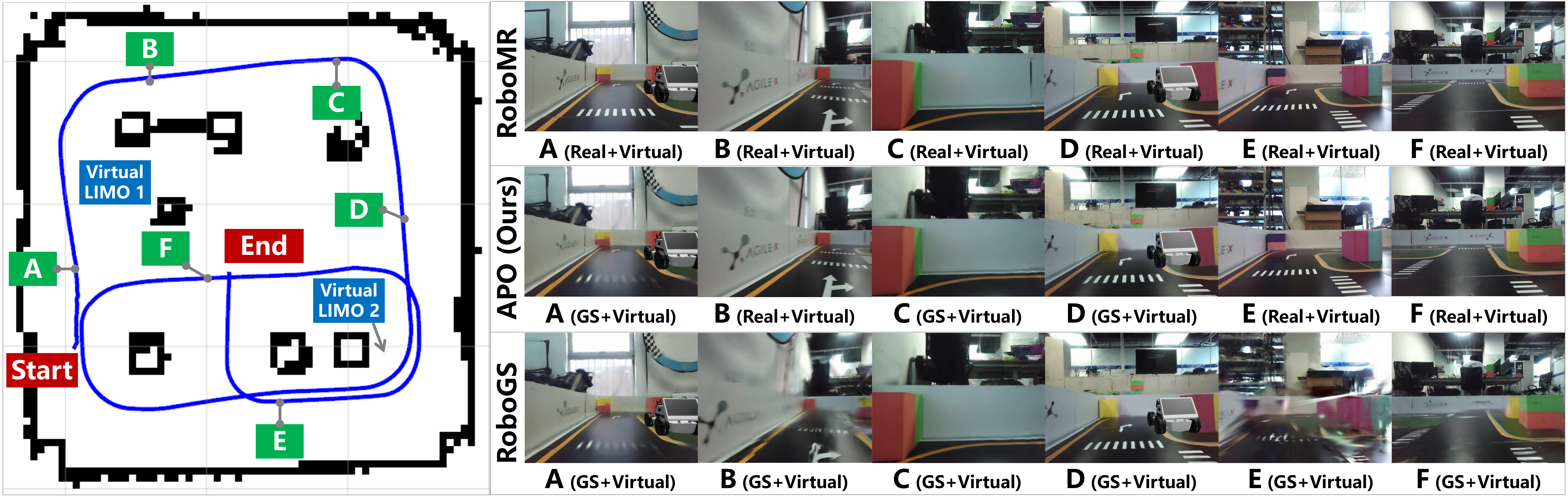}
		\caption{Camera positions and the associated qualitative comparison of images.}
	\end{subfigure}
 	\begin{subfigure}{0.3\linewidth}  
		\centering
		\includegraphics[width=1\linewidth]{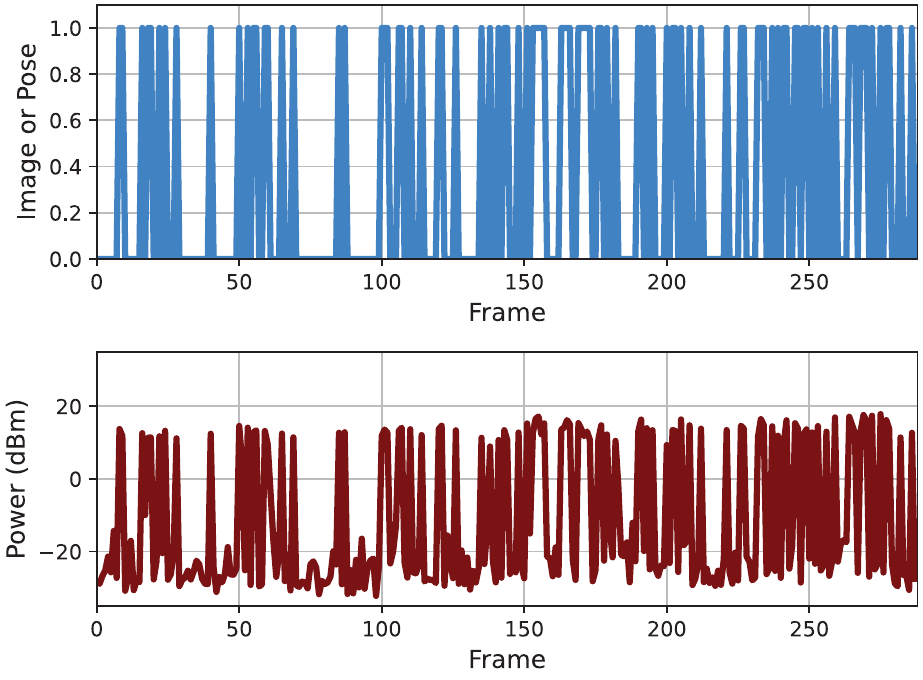}
		\caption{Content switching and power profile.}
	\end{subfigure}
    \vspace{-0.1in}
	\caption{(a) Left-hand side: The MR scenario, where the blue line represents robot trajectory, red boxes represent start and end positions, green boxes represent image poses, and blue boxes represent positions of virtual agents, respectively. Right-hand side: Qualitative comparison of 6 image frames. (b) Content switching $\{x_t\}$ and power profiles $\{p_t\}$.}
    \label{fig:demo}
\end{figure*}

\section{Conclusion}\label{section6}

This paper presented GSRMR, which realizes low-energy communication between the simulation server and robot terminal via the introduction of GS. 
The GSCLO framework was proposed to avoid GS discrepancies, which jointly optimizes the content switching and power allocation via a fast APO algorithm.
Various experiments were conducted, which demonstrate the effectiveness of GSRMR and the superiority of APO over extensive benchmarks.

\input{ref.bbl}

\end{document}